\documentclass{article}

\usepackage{arxiv}
\usepackage{graphicx, amsmath, pgfplots, amsfonts, hyperref, setspace, natbib, amsthm}
\usepackage{algorithm,algpseudocode}
\algnewcommand\algorithmicforeach{\textbf{for each}}
\algdef{S}[FOR]{ForEach}[1]{\algorithmicforeach\ #1\ \algorithmicdo}

\DeclareMathOperator*{\argmin}{argmin}
\renewcommand{\vec}[1]{\mathbf{#1}}
\pgfplotsset{compat=1.16}

\newtheorem{theorem}{Theorem}
\newtheorem{lemma}[theorem]{Lemma}

\title{A Data Science Approach to Risk Assessment for Automobile Insurance Policies}

\author{ \href{https://orcid.org/0000-0003-1729-559X}{\includegraphics[scale=0.06]{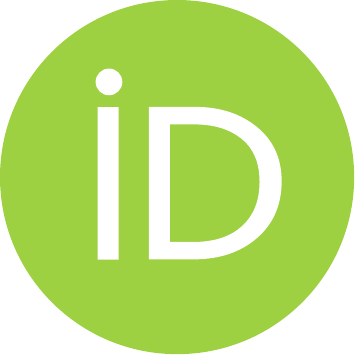}\hspace{1mm}Patrick Hosein}\\
	The University of the West Indies,
	St. Augustine, Trinidad\\
	\texttt{patrick.hosein@sta.uwi.edu} 
}

\renewcommand{\shorttitle}{Automobile Insurance Policy Risk Assessment}

\hypersetup{
pdftitle={A Data Science Approach to Risk Assessment of Automobile Insurance Policies},
pdfauthor=Patrick Hosein,
colorlinks=true
}

\begin{document}

\maketitle

\begin{abstract}
In order to determine a suitable automobile insurance policy premium one needs to take into account three factors, the risk associated with the drivers and cars on the policy, the operational costs associated with management of the policy and the desired profit margin. The premium should then be some function of these three values. We focus on risk assessment using a Data Science approach. Instead of using the traditional frequency and severity metrics we instead predict the total claims that will be made by a new customer using historical data of current and past policies. Given multiple features of the policy (age and gender of drivers, value of car, previous accidents, etc.) one can potentially try to provide personalized insurance policies based specifically on these features as follows. We can compute the average claims made per year of all past and current policies with identical features and then take an average over these claim rates. Unfortunately there may not be sufficient samples to obtain a robust average. We can instead try to include policies that are``similar" to obtain sufficient samples for a robust average. We therefore face a trade-off between personalization (only using closely similar policies) and robustness (extending the domain far enough to capture sufficient samples). This is known as the Bias-Variance Trade-off. We model this problem and determine the optimal trade-off between the two (i.e. the balance that provides the highest prediction accuracy) and apply it to the claim rate prediction problem. We demonstrate our approach using real data.
\end{abstract}


\section{Introduction}

Traditionally insurance companies have determined automobile policy premiums using rate tables computed by Actuaries \cite{Hassani2020BigDA}. Today, however, the vast amount of data collected in electronic form can now be used to determine more suitable premiums for a given policy since such data can be used to better predict risk \cite{errais2019pricing}. Furthermore, by using data from present and past customers, the predictions are better suited for the particular environment in which the insurance company operates. This form of personalized policies benefit the customer (who pays an amount more in line with their risk) as well as the insurance company (which can now better ensure that it can safely cover claims costs from risky policies). The typical approach is straightforward. For a given new customer, one can use historical data of past and present customers with similar characteristics (features) to better estimate the risk level of the new customer and then use this to determine a premium for their policy. This is similar to recommender systems used by companies such as Netflix. In that case movies are recommended to an individual based on movies that were enjoyed by customers with similar characteristics (collaborative filtering). In the case of insurance, one must recommend a policy that is both desirable to the customer (through personalization) and profitable to the insurance provider.

\section{Related Work and Contributions}

Many past papers have focused on Recommender systems for insurance companies where one of a small number of insurance products is offered. In \cite{https://doi.org/10.1002/widm.1363, 10.1145/3109859.3109907} they used historical data of existing and past customers to determine the most suitable policy for a new customer. In this case a relatively small number of insurance products 
are available and hence the number of customers who have been using a specific product will be sufficiently high so that the sample size is not an issue when computing the recommendation. The paper \cite{8614094} also addresses the same problem but focuses on speeding up the computation of the recommendations. The papers \cite{8622491, BIAN201820} do address personalized auto insurance premiums but they focus on using Telematics data to do so. Such devices are not available from all insurance companies and hence has limited applicability. The authors in \cite{doi:10.1177/0972150920932287} use Fuzzy Logic to come up with a rule based approach to risk. In our case we use a data science approach and focus on personalization while using traditionally available data.

Several papers also focus on risk assessment. In general, few customers make claims during a year. Furthermore, the claims that are made vary widely from minor incidents (such as a scratched bumper) to major ones (such as when a car has to be written off because it cannot be repaired). This results in a large variation in the average annual claims made by a customer making it difficult to predict. Therefore papers generally focus on predicting either, severity (the expected claim value given that a claim is made, \cite{10.1145/3109859.3109907, su2020stochastic}) or frequency (the expected number of claims made per year, \cite{liu2014using, david2015modeling}). Another typical metric is the loss ratio which is the ratio of the claims made to the premium charged (\cite{guelman2012gradient, zhang2013predicting}. We focus on a more direct measure which is the average total value of claims made per year for a given policy (shortened to simply claims rate) which can be thought of as the product of severity and frequency and hence captures both metrics. As mentioned before, predicting claims rate can be challenging because of its high variance. A significant number of samples are required for a good estimate but as one tries to achieve greater personalization the number of available samples decreases. We investigate the optimal trade-off between these objectives.

Note one potential issue of Recommender Systems is the following. The Recommender system chooses the most appropriate product for the customer but this may not be a very profitable product for the company and so this trade-off must be taken into account (see \cite{hosein2019recommendations} for a more detailed discussion on this issue). In our case we need not worry about this issue since we are focusing on providing the most suitable (unique) product and the premium is then determined to achieve an acceptable profit for each unique policy.

Our contribution is in the analysis of the trade-off between personalization and robustness. Instead of a finite number of products from which to choose for customer offerings, we provide a unique (personalized) product to each customer. Furthermore, we only take into account demographic and other data collected from each policy but do not consider Telematics data. Naturally our approach can include such data as well. Therefore we are (a) using a new model for risk (based on claim rate), (b) obtaining the best trade-off between personalization and robustness, (c) using the proposed approach for feature importance and selection and (d) demonstrating how the results obtained can be interpreted so that one can explain to the customer the reason for the provided premium.

\section{Problem Formulation and Assumptions}

We formulate a model for this problem and then develop an algorithm for its solution.
Our objective is as follows, given policy information for a new or renewal automobile policy (i.e., information about the drivers, cars, etc.), predict the expected total amount in claims that will have to be paid out to this customer over the subsequent year of the policy. This prediction will be based on several factors but correlates with the risk associated with the drivers and cars on the policy. This information can then be used to determine an appropriate premium for the policy. Traditionally this computation is performed using risk tables but independent of the specific historical data of the company's customers. Here we use historical data of the provider's customers to make the prediction. This is more appropriate since the parameters used in the risk tables may have been developed based on a different customer base (country) and so unsuitable for the one under consideration.

\subsection{New versus Renewal Policies}

Note that we need to distinguish between a new policy, for which only customer provided data is available, and renewals, for which information about the customer since the start of their policy is available. We develop a model that can be applied to both new and renewal policies. In the case of renewal policies, the historical data of the policy is included in the training set. The proposed approach therefore automatically includes the past claim information of the policy (since it is now included in the training set). Therefore we assume that all policies that are at least one year old are included in the training set used for parameter determination. In this way recent information is included in the predictions. Note that this means there is no need for an accident penalty or a no-claim discount since these adjustments are implicit.

\subsection{Quantity versus Currency of Data}

The more data that is used for predictions, the more accurate the prediction. However, as we increase the dataset by going further back in time we will be using outdated information (e.g., automobile models, cost of repairs, etc.). We manage this as follows. As the cost of claims increases (with time), the claim rate of a policy will also increase. The prediction we get from using outdated information will therefore be lower than what would actually occur. We therefore scale predictions as follows. We predict the total claims for the previous year and we then use a scaling factor to ensure that the total predicted claims equals the total actual claims. This scaling factor is then included when making new predictions. This scale factor computation is repeated every year so that the total predicted claims for the upcoming year will be close to the actual total claims for the year. 

\subsection{Comprehensive versus Third-Party Policies}

There are two types of policies, Comprehensive, in which the company has to pay for repairs to the customer's car even if they were at fault, and Third-Party, in which the company only pays for repairs to the other involved party in the accident (i.e., the third party). Note that the risk behaviour (and claims requests) of Third-Party versus Comprehensive policy customers may be different but the approach we use has the ability to extract the relevant information. We therefore make predictions using the combined dataset (i.e. policies of both types) but include the type of policy as a feature.
Note that the features for both types of policies are the same except that, for Comprehensive policies, there is also the Sum Insured (based on the value of car) feature. This value is set to zero for Third Party policies but the same model can be used for both policy types.

\subsection{Multi-Car versus Single-Car Policies}

For each policy we must predict the total annual claims for the policy which may have multiple drivers and/or cars. Note that a premium is charged per car and the sum of these forms the policy premium. Our model uses the primary car and primary driver of that car as the sample for that policy (and ignores all other drivers/cars). This means that the prediction is made for a single driver/car pair and this can be repeated for each car on the policy to determine the total claim rate for the policy. 

\section{Dataset Description and Preparation}

The policy data used for this study spans a period of 5 years. No confidential information is disclosed and all monetary values are normalized. It consists of data collected from past and existing customers. Each policy record consists of policy information, information for each driver on the policy, information for each vehicle on the policy  and information on each claim made on the policy since its inception. Some of this information is not relevant for our purposes (e.g., Vehicle Identification Number) and is ignored. Certain features must be derived from the information provided. For example, the policy lifetime is computed as the difference between the termination date and start date (if terminated) or the difference between the present date and the start date (if currently active).
Note that the metric of concern is the average claim rate for a driver/car pair. For each policy we determine the total value of all claims made (by the primary driver) and divide by the total lifetime of the policy (in years) to obtain the claim rate. 

Our objective is to predict the claim rate and use this claim rate to determine a suitable price. In order to do this we focus only on the primary driver and their associated car for each policy. This happens to be the majority of cases so we do not lose too much information. For this driver we compute the claim rate based on accidents in which they were involved. We remove features that were mostly empty or corrupt and also placed filters to remove anomalous data such as drivers over the age of 85. The data that was finally used for the problem is provided in Table \ref{claims}. POL is the policy number which is used as a unique identifier for the policy. CLR is the average claims per year computed for the primary driver and their associated car for the policy. TOC is the type of policy (customer) which we also use as a feature. SIV is the sum insured value of the primary vehicle of the policy and this value is zero for Third Party policies. All other features are described in the Table.

\begin{table}
\caption{Policy Features used for Analysis}
\label{claims}
\centering
\setlength{\tabcolsep}{10pt}
\renewcommand{\arraystretch}{1.5}
{\small
\begin{tabular}{|c|l|} \hline
{\bf Feature} & {\bf Description} \\ \hline
POL & policy identification number	\\ \hline
CLR & annual claim rate (total claims divided by policy lifetime)	\\ \hline
ADR & city of home address\\ \hline
COV	& were drivers continuously insured over the last 5 years? (y/n)	\\ \hline
SEX	& gender of driver\\ \hline
AGE & age of driver \\ \hline
MST & marital status of driver\\ \hline
USE & type of use (business, work or pleasure) \\ \hline
WRK & whether car is used for work (y/n) \\ \hline
NAF & the number of at-fault accidents over the last 5 years \\ \hline
DAF & number of years primary driver has been free of Claims \\ \hline
NNF & the number of not-at-fault accidents over last 5 years of driver\\ \hline
MAK & car manufacturer \\ \hline
VYR & model year \\ \hline
BDY & body type \\ \hline
YCF & the number of years the Primary Car has been claim free \\ \hline
NCC & engine size of car (in CC)\\ \hline
TOC & type of policy (Comprehensive or Third Party) \\ \hline
SIV & sum insured value\\ \hline
\end{tabular}
}
\end{table}

\section{Proposed Model}

The model we propose is unique in that (a) the metric of concern is claim rate and (b) we use a novel solution approach rather than the traditional approaches. We do not present a full comparison with other Machine Learning approaches in this paper since our intent is to introduce the model. Future papers will include detailed comparisons with state of the art Machine Learning algorithms.

\subsection{Definition of Distance Metric}

In this section we describe the approach used for predicting the annual financial claims per year (henceforth called claim rate) for a given policy. We denote the set of features that we consider by the set $\vec{F}$. Features include information such as age, gender, etc., as well as information about their associated vehicle such as model, body type, etc.
We denote the set of samples by $\vec{S}$ where a sample is a policy  and includes features for the associated driver/car pair.
One way to predict the claim rate is to find the expected value of the claim rates of all existing policies with identical features. However, there may be none or very few of such policies. We must therefore include policies with features that are nearby and include them in the average.

In order to find ``close" policies we need to define a distance metric between pairs of categories of a given feature and then use some measure (e.g., Euclidean Distance) to define the distance between two policies. We define this distance as follows. For each category $v$ of feature $f$ let $C(f,v)$ denote the claim rate averaged over all policies that has a value $v$ for feature $f$. For example, for the feature gender ($f=gender$) with members $m$ and $f$, let $C(gender,male)$ denote the average claim rate over all male drivers and let $C(gender,female)$ denote the average claim rate over all female drivers. We define the distance between these two categories of this feature by $|C(gender,male) - C(gender,female)|$. In general, if we had several feature categories then the distance between any two of them will be computed in this manner. Therefore if the test policy has a male driver then their gender distance from another policy with a male driver is 0 while for a female driver it would be $|C(gender,male) - C(gender,female)|$. Note that the same computation is done for numerical features such as age. For example, the distance between a 48 year old and a 30 year old is given by $|C(age,48) - C(age,30)|$. By doing this we maintain the same measurement unit (claim rate) for all distances. If the 48 year old is a male and the 30 year old is a female then the Euclidean distance is used (i.e. the root of the sum of the squares of the gender and age feature distances). 

\subsection{Claim Rate Prediction}

If there were several existing policies with the exact feature values as the test policy then one could obtain a good estimate on the claim rate for the test policy by taking the average of claim rates over all policies with the same features. However, in general there may not be sufficient samples (or none) to obtain an estimate with sufficient confidence and so we need to include nearby samples as well. The more nearby samples we use the more robust the estimate but the less personalized since included samples are further away. This in turn leads to lower prediction accuracy. We take a weighted average of claims of all policies where the weight is inversely proportional to the Euclidean distance between the policies.

Suppose we wish to predict the claim rate for some test policy and denote the distance between this policy and some training policy $s$ by $d_s$.
We use a weight $(1 + d_s)^{-\kappa}$ for $\kappa \ge 0$ when taking into account the claim rate of sample $s \in \vec{S}$. However we need to have a normalizing factor $\alpha$. The predicted claim rate $c$ for the test sample is therefore given by
\begin{equation}
    c(\kappa) \equiv \sum_{s \in \vec{S}} \alpha \frac{c_s}{(1 + d_s)^{\kappa}}
\end{equation}
where $c_s$ is the claim rate of policy $s$.
If all policies had the same claim rate then the predicted claim rate should also have this value and hence we must have
\begin{equation}
    c \equiv \sum_{s \in \vec{S}} \alpha \frac{c}{(1 + d_s)^{\kappa}}
\end{equation}
and hence
\begin{equation}
    \alpha = \left( \sum_{s \in \vec{S}} \frac{1}{(1 + d_s)^{\kappa}} \right)^{-1}
\end{equation}
and so we have the predicted claim rate for the test policy as
\begin{equation}
    c(\kappa) \equiv \frac{\sum_{s \in \vec{S}} \frac{c_s}{(1 + d_s)^{\kappa}}}{\sum_{s \in \vec{S}} \frac{1}{(1 + d_s)^{\kappa}}}
\end{equation}
The pseudo-code for this computation is provided in Algorithm \ref{alg}.

\begin{algorithm}
\caption{Pseudo-code for proposed Algorithm to predict test sample claim rate $c(\kappa)$}
\label{alg}
\setstretch{1.6}
\begin{algorithmic}[1]
\State $\vec{F} \equiv \text{\sf set of features}$
\State $\vec{S} \equiv \text{\sf set of training samples}$
\State $\vec{v}_f \equiv \text{\sf set of categories for feature $f \in \vec{F}$}$
\State $\kappa > 0$ \text{tuning parameter}
\State $X_{sf} \in \vec{v}_f \equiv \text{\sf category of feature $f \in \vec{F}$ of training sample $s \in \vec{S}$}$
\State $x_f \in \vec{v}_f \equiv \text{\sf category of feature $f \in \vec{F}$ for test sample}$
\State $c_s \equiv \text{claim rate for sample $s \in \vec{S}$}$
\State $c(\kappa) \equiv \text{predicted claim rate for test sample using parameter $\kappa$}$
\State $\displaystyle \bar{c} \leftarrow \frac{1}{|\vec{S}|} \sum_{s \in \vec{S}} c_s \;\;\; \text{(average claim rate over all training samples)}$
\ForEach {$f \in \vec{F} $}
\ForEach {$v \in \vec{v}_f $}
\State $\vec{z} \equiv \{s \in \vec{S}\;|\;X_{sf} = v\}$
\State $\displaystyle C(f, v) \leftarrow \frac{1}{|\vec{z}|} \sum_{s \in \vec{z}} c_s \;\;\; \text{(average claim rate over samples where feature $f$ has value $v$)}$
\EndFor
\EndFor
\ForEach {$s \in \vec{S}$}
\State $\displaystyle d_s \leftarrow  \left( \sum_{f \in \vec{F}} \left( C(f,x_f) - C(f,X_{sf}) \right)^2 \right) ^{\frac{1}{2}}$
\State $d_s \leftarrow \frac{d_s}{\bar{c}}$
\EndFor
\State $\displaystyle c(\kappa) = \frac{\sum_{s \in \vec{S}} \frac{c_s}{(1 + d_s)^{\kappa}}}{\sum_{s \in \vec{S}} \frac{1}{(1 + d_s)^{\kappa}}}$ \;\;\; \text{(predicted claim rate for test policy)}
\end{algorithmic}
\end{algorithm}

\subsection{Computing the Optimal value of \texorpdfstring{$\kappa$}{}}

Next we determine the optimal value of $\kappa$. For an existing policy $s$ we have the actual claim rate $c_s$. Note that we can predict a claim rate for this sample (in which case the sample must be removed from the training set) and we denote this predicted value by $\hat{c}(\kappa)$. We introduce the hat to distinguish this predicted value with the actual value (which has no hat). Note that we use 5-Fold cross validation and hence 80\% of the samples are used for training (computing the average claim rates $C(f,v)$) while the other 20\% are used for testing (and determination of the accuracy).
Note that when $\kappa=0$ then $\hat{c}(\kappa) = \bar{c}$ and so the prediction is simply the average over all (training) samples. As $\kappa$ is increased, close samples are weighted more heavily but the average becomes less robust and hence the error will eventually start increasing again. Therefore the optimal $\kappa$ lies somewhere in between (see Figure \ref{kap} for an example of this relationship). We therefore will find $\kappa$ that minimizes the Mean Absolute Error (MAE) of the prediction. For convenience we will normalize this by the MAE if one used the average claim rate over all policies, $\bar{c}$, as the predictor. One can think of this case as making a prediction without features. Therefore we will compare the error of the prediction made with features with the error of the prediction made without features. Let us denote the test set by $\vec{T}$ then we compute the normalized error over the test samples as
\begin{equation}
    E(\kappa) = \frac{\sum_{t \in \vec{T}} |\hat{c}_t(\kappa) - c_t|}{\sum_{t \in \vec{T}} |\bar{c} - c_t|}
\end{equation}
If the predictor is the same as averaging over all policies (i.e., $\hat{c}_t(\kappa) = \bar{c}$) then this ratio is 1. However if, by adding features, the MAE of the predictor is decreased then this ratio drops below 1.
Therefore this metric provides an indication of prediction performance using features when compared to prediction performance without using features and hence demonstrates the benefit of the feature-based approach. We then find the $\kappa$ value that optimizes the predictor as
\begin{equation}
    \kappa^* = \argmin_{\kappa} E(\kappa)
\end{equation}
This value is then used to obtain the optimal prediction as $c_t^* = \hat{c}_t(\kappa^*)$.

\subsection{Feature Importance}

Consider a single feature. We know from the previous section that, as $\kappa$ is increased then $E(\kappa)$ should initially decrease before increasing once again. If this does not occur then the feature does not capture sufficient information to be useful for predictions. One can therefore use the value of $E(\kappa)$ evaluated at the optimal $\kappa$ for that feature alone as an indication of importance. In fact, even if we used a fixed value of $\kappa$ for each feature the corresponding value of $E(\kappa)$ is an indication of relative importance with lower values indicating more importance. For example, in Figure \ref{imp} we plot $E(\kappa)$ as a function of $\kappa$ for two features DAF (years claim free for the driver) and YCF (years claim free for the car). For DAF the minimum error occurs at $\kappa = 8$ while for YCF it occurs at $\kappa >20$. However, at $\kappa=10$ we find that the respective values provide a good representation of the optimal value and hence can be used to compare the two features. Also note that here we clearly see that risk depends primarily on the driver with the car playing a minor role.

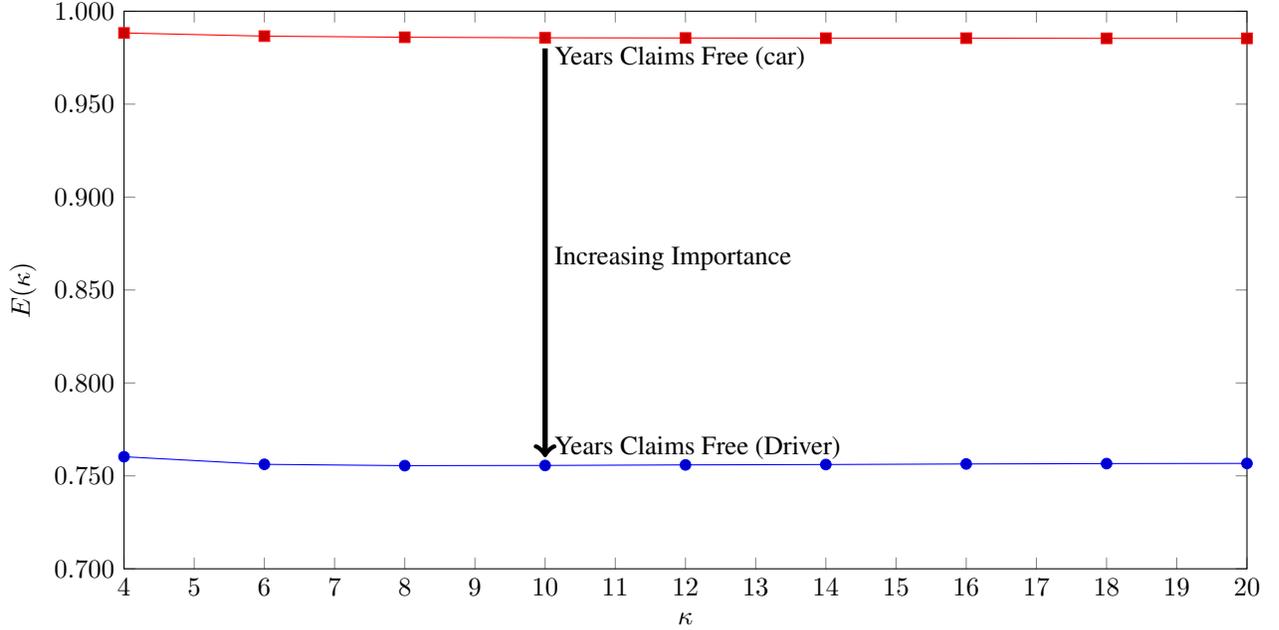
\begin{figure}
\begin{tikzpicture}
\begin{axis}[width=\textwidth, height=9cm, xlabel=$\kappa$, mark=none, ylabel=$E(\kappa)$, xmin=4, xmax=20, ymin=0.7, ymax=1,
y tick label style={
        /pgf/number format/.cd,
            fixed,
            fixed zerofill,
            precision=3,
        /tikz/.cd
    },
]
\addplot coordinates {
(4, 0.7603) (6, 0.7562) (8, 0.7555) (10, 0.7556) (12, 0.7559) (14, 0.7561) (16, 0.7564) (18, 0.7566) (20, 0.7567)};
\addplot coordinates {
(4, 0.9883) (6, 0.9866) (8, 0.9860) (10, 0.9857) (12, 0.9856) (14,  0.9855) (16, 0.9855) (18, 0.9854) (20, 0.9854)};
\draw[->, line width=2pt] (axis cs:10, 0.98) -- (axis cs:10, 0.76);
\node[right] at (10, 0.975) {Years Claims Free (car)};
\node[right] at (10,0.867) {Increasing Importance};
\node[right] at (10,0.765) {Years Claims Free (Driver)};
\end{axis}
\end{tikzpicture}
\caption{$E(\kappa)$ of two Features to demonstrate Relative Importance}
\label{imp}
\end{figure}

We therefore use this approach to determine which features are important and hence should be included in the analysis. We use a value of $\kappa=10$ and, using a single feature at a time, we compute $E(10)$. The resulting values are provided in Figure \ref{infogain}. The features represented in red have normalized errors greater than 1. 

\begin{figure}
\centering
\begin{tikzpicture}
 \begin{axis}[
 axis on top,
  ticklabel style = {font=\scriptsize},
  xbar, xmin=0,
  bar width= 0.1cm, y dir=reverse,
  width=\textwidth, height= 8cm, enlarge y limits=0.01,
  xlabel={$E(10)$},
  y=0.6cm,
  symbolic y coords={a, DAF, NAF, NNF, TOC, YCF, VYR, BDY, COV, NCC, MAK, USE, MST, WRK, ADR, SEX, AGE, SIV,  DAF+NAF, DAF+NAF+NNF, DAF+NAF+NNF+TOC, DAF+NAF+NNF+TOC+YCF, DAF+NAF+NNF+TOC+YCF+VYR, DAF+NAF+NNF+TOC+YCF+BDY, DAF+NAF+NNF+TOC+YCF+COV, DAF+NAF+NNF+TOC+YCF+COV+NCC, DAF+NAF+NNF+TOC+YCF+COV+MAK, DAF+NAF+NNF+TOC+YCF+COV+USE, DAF+NAF+NNF+TOC+YCF+COV+USE+MST, DAF+NAF+NNF+TOC+YCF+COV+USE+WRK, DAF+NAF+NNF+TOC+YCF+COV+USE+ADR, DAF+NAF+NNF+TOC+YCF+COV+USE+SEX, DAF+NAF+NNF+TOC+YCF+COV+USE+SEX+AGE, DAF+NAF+NNF+TOC+YCF+COV+USE+SEX+SIV, b},
  ytick={DAF, NAF, NNF, TOC, YCF, VYR, BDY, COV, NCC, MAK, USE, MST, WRK, ADR, SEX, AGE, SIV, DAF+NAF, DAF+NAF+NNF, DAF+NAF+NNF+TOC, DAF+NAF+NNF+TOC+YCF, DAF+NAF+NNF+TOC+YCF+VYR, DAF+NAF+NNF+TOC+YCF+BDY, DAF+NAF+NNF+TOC+YCF+COV,  DAF+NAF+NNF+TOC+YCF+COV+NCC, DAF+NAF+NNF+TOC+YCF+COV+MAK, DAF+NAF+NNF+TOC+YCF+COV+USE, DAF+NAF+NNF+TOC+YCF+COV+USE+MST, DAF+NAF+NNF+TOC+YCF+COV+USE+WRK, DAF+NAF+NNF+TOC+YCF+COV+USE+ADR, DAF+NAF+NNF+TOC+YCF+COV+USE+SEX, DAF+NAF+NNF+TOC+YCF+COV+USE+SEX+AGE, DAF+NAF+NNF+TOC+YCF+COV+USE+SEX+SIV}, ytick style={draw=none},
    y tick label style={anchor=west,color=black,xshift= \pgfkeysvalueof{/pgfplots/major tick length}, yshift=3mm},
  bar shift=0pt,
 ]
 0.97927
\addplot coordinates {(0,a)};
\addplot [color=blue, fill] coordinates {(0.7556,DAF)};
\addplot [color=blue, fill] coordinates {(0.7987,NAF)};
\addplot [color=blue, fill] coordinates  {(0.9708,NNF)};
\addplot [color=blue, fill] coordinates {(0.9793,TOC)};
\addplot [color=blue, fill] coordinates {(0.9848,YCF)};
\addplot [color=blue, fill] coordinates {(0.9950,VYR)};
\addplot [color=blue, fill] coordinates {(0.9974,BDY)};
\addplot [color=blue, fill] coordinates {(0.9985,COV)};
\addplot [color=blue, fill] coordinates {(0.9986,NCC)};
\addplot [color=blue, fill] coordinates {(0.9989,MAK)};
\addplot [color=blue, fill] coordinates {(0.9989,USE)};
\addplot [color=blue, fill] coordinates {(0.9991,MST)};
\addplot [color=blue, fill] coordinates {(0.9994,WRK)};
\addplot [color=red, fill] coordinates {(1.000,ADR)};
\addplot [color=red,fill] coordinates {(1.000,SEX)};
\addplot [color=red,fill] coordinates {(1.001,AGE)};
\addplot [color=red,fill] coordinates {(1.009,SIV)};
\addplot [color=blue,fill] coordinates {(0.7549,DAF+NAF)};
\addplot [color=blue,fill] coordinates {(0.6840,DAF+NAF+NNF)};
\addplot [color=blue,fill] coordinates {(0.6654,DAF+NAF+NNF+TOC)};
\addplot [color=blue,fill] coordinates {(0.6403,DAF+NAF+NNF+TOC+YCF)};
\addplot [color=brown,fill] coordinates {(0.6464,DAF+NAF+NNF+TOC+YCF+VYR)};
\addplot [color=brown,fill] coordinates {(0.6443,DAF+NAF+NNF+TOC+YCF+BDY)};
\addplot [color=blue,fill] coordinates {(0.6375,DAF+NAF+NNF+TOC+YCF+COV)};
\addplot [color=brown,fill] coordinates {(0.6428,DAF+NAF+NNF+TOC+YCF+COV+NCC)};
\addplot [color=brown,fill] coordinates {(0.6391,DAF+NAF+NNF+TOC+YCF+COV+MAK)};
\addplot [color=blue,fill] coordinates {(0.6366,DAF+NAF+NNF+TOC+YCF+COV+USE)};
\addplot [color=brown,fill] coordinates {(0.6404,DAF+NAF+NNF+TOC+YCF+COV+USE+MST)};
\addplot [color=brown,fill] coordinates {(0.6380,DAF+NAF+NNF+TOC+YCF+COV+USE+WRK)};
\addplot [color=brown,fill] coordinates {(0.6898,DAF+NAF+NNF+TOC+YCF+COV+USE+ADR)};
\addplot [color=blue,fill] coordinates {(0.6356,DAF+NAF+NNF+TOC+YCF+COV+USE+SEX)};
\addplot [color=brown,fill] coordinates {(0.6444,DAF+NAF+NNF+TOC+YCF+COV+USE+SEX+AGE)};
\addplot [color=brown,fill] coordinates {(0.6958,DAF+NAF+NNF+TOC+YCF+COV+USE+SEX+SIV)};
\addplot [color=pink,fill] coordinates {(0,b)};
\draw [dashed] (axis cs:1,b) -- (axis cs:1,a);
\end{axis}
\end{tikzpicture}
\caption{Normalized MAE for proposed predictor computed for each feature}
\label{infogain}
\end{figure}
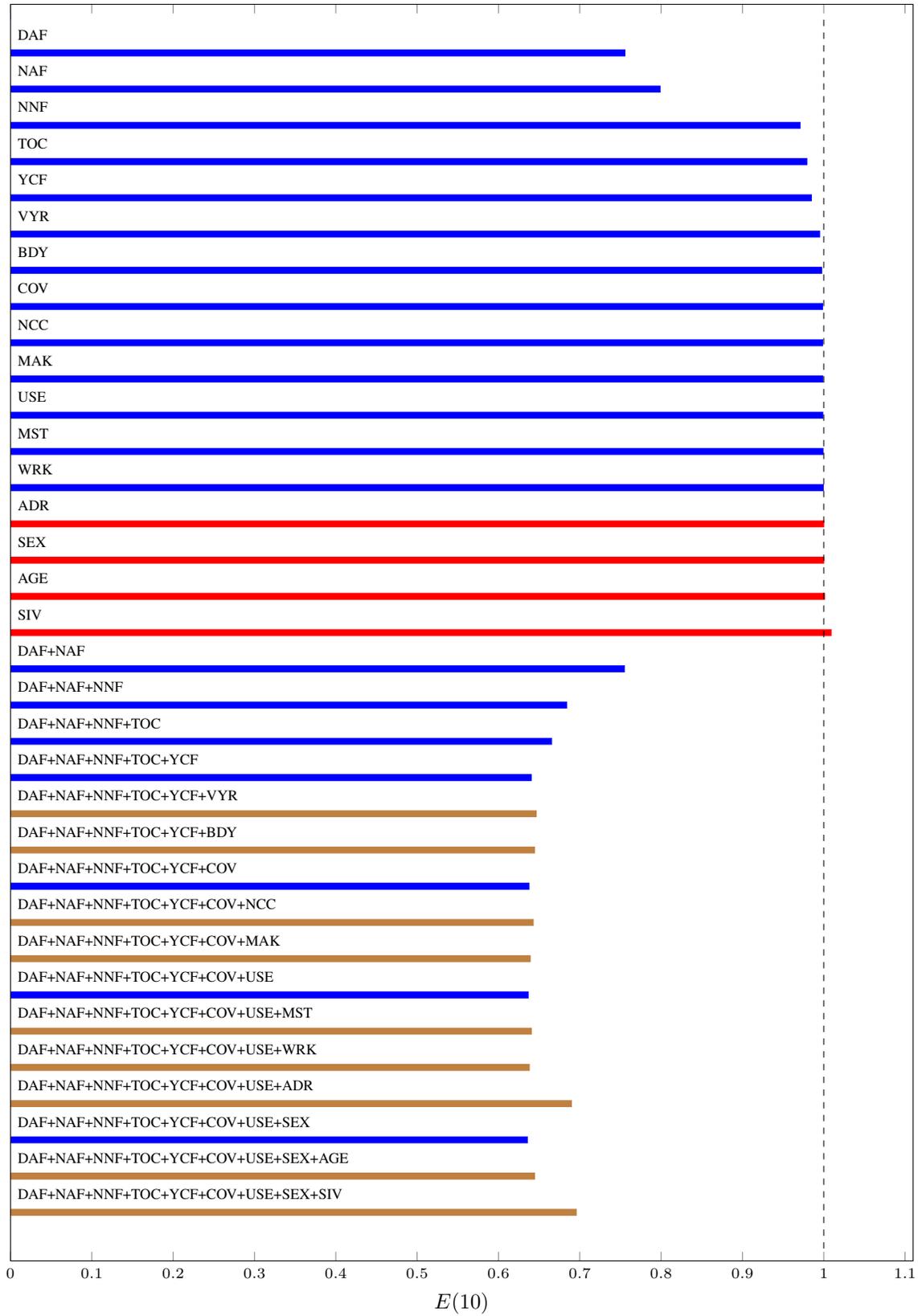

\subsection{Feature Selection}

Now that we know which features are important we focus on which of them should be included in the model. We do this as follows. Starting from the most important feature (lowest value for $E(10)$) we add one feature at a time and again compute $E(10)$ for the combination. If the performance metric decreases (i.e., better results) then we keep it and repeat. If the performance metric increases then we remove the recently added feature and repeat. Note that, some features may have low importance when considered in isolation but together with other features (such as type of policy) their value increases. The results of this process are provided in the lower part of Figure \ref{infogain}. A brown bar indicates that the addition of a feature resulted in a loss of performance and hence the feature should be removed.

The final features to be used include number of years since the driver was last in an accident, the number of at-fault accidents by the driver over the last five years and the number of not-at-fault accidents over the last five years. Each of these is a strong indicator of risk. In the case of renewals we would have actual claim rate values but for new customers these three features (even without financial information) correlates well with claim rate. The type of policy feature is also needed since it helps to distinguish the two types of policies. Note that we could do the analysis separately for each type of policy but the increase in sample size by combining the two types provides better overall results. Only one car feature was found to be sufficiently beneficial and that was the number of years since the last claim was made on the car. However this feature is far less important than the driver features that were included indicating that what really matters is the driver on the policy and not the car. Whether the driver was continually insured over the last 5 years (i.e., mature driver), the type of use (personal versus business) as well as the gender of the driver were also found to be useful (but far less so than the others). 

Note that we had expected certain features (like age) to be beneficial but they were not. In Figure \ref{age} we provide a histogram of the average claim rate by age (in blue). We see that there is a weak dependency on age but because of the large variations from year to year (because of limited data), the dependency is not sufficiently robust. Next we predicted the claim rate for each age using the approach described previously. We found that the optimal value of $\kappa$ was 2 with a normalized MAE of $E(2) = 0.9996$ which indicates that limited personalization was possible. We then used this value to find optimal claim rate values for each age. In Figure \ref{age} we provide the histogram of the original claim rates (blue) and the filtered claim rates (red). We note that the red claim rates are each close to unity and hence provides little differentiation. This is why this feature does not provide much benefit for predictions.

\begin{figure}
\begin{tikzpicture}
\begin{axis} [
ybar, /pgf/bar width=1pt,
ticklabel style = {font=\tiny}, ymin=0, xmin=17, xmax=85,
width=\textwidth, height=8cm,
]
\addplot table {age.data};
\addplot table {age-filter.data};
\end{axis}
\end{tikzpicture}
\label{age}
\caption{Histogram of Claim Rate versus Age for Original (blue) and filtered (red) Cases}
\end{figure}
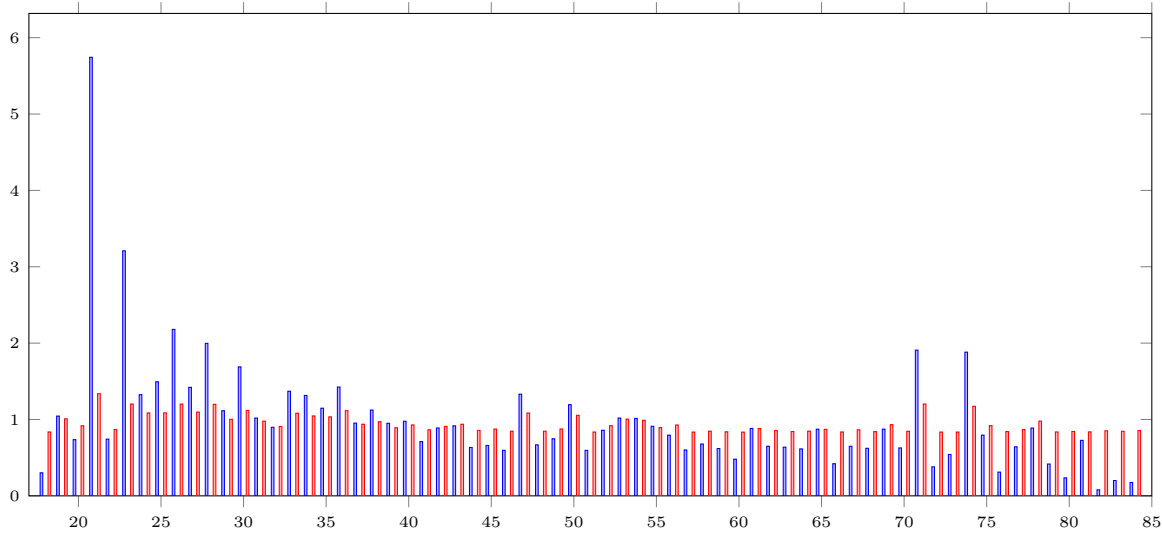

\subsection{Parameter Optimization}

We now have the set of features to be included in the model. Next we find the optimal value of $\kappa$ for this combination of features. This value will then be used for making predictions. In Figure \ref{kap} (brown curve) we plot $E(\kappa)$ as a function of $\kappa$. We find the optimal value to be $\kappa^* = 8$ with $E(8) = 0.63$ and hence once can reduce the MAE obtained with no features by 37\% by using the using the 8 chosen features. We also note that, although $E$ increases with $\kappa$ beyond the optimal point, the increase is very gradual so the error remains nearly constant for a wide range of values and so the approach is robust with respect to $\kappa$.

We believe that if we had performed feature selection using each policy type separately that we would get the same features. We therefore used these features and determined $E(\kappa)$ for Third-Party policies only and also for Comprehensive policies only. These are also plotted in Figure \ref{kap}. We find that the accuracy for Third-Party only samples is close to that of the case of using both Third-Party and Comprehensive samples. This is primarily due to the fact that there are 50\% more Third-Party samples than Comprehensive samples. Therefore the Third-Party samples are more useful to the Comprehensive predictions than the other way around. Also note that all three cases are optimal at $\kappa=8$.

\begin{figure}
\begin{tikzpicture}
\begin{axis}[width=\textwidth, height=9cm, xlabel=$\kappa$, mark=none, ylabel=$E(\kappa)$, xmin=5, xmax=11, ymin=0.630, ymax=0.670,
y tick label style={
        /pgf/number format/.cd,
            fixed,
            fixed zerofill,
            precision=3,
        /tikz/.cd
    },
]
\addplot coordinates {
(5, 0.670) (6, 0.66135) (7, 0.65733) (8, 0.65629) (9, 0.65700) (10, 0.65868) (11, 0.66076)};
\addplot coordinates {
(5, 0.64307) (6, 0.63960) (7, 0.63822) (8, 0.63785) (9, 0.63804) (10, 0.63859) (11, 0.63928)};
\addplot coordinates {
(5, 0.65054) (6, 0.64097) (7, 0.63655) (8, 0.63498) (9, 0.63505) (10, 0.63560) (11, 0.63634)};
\legend{Comprehensive, Third-Party, Comprehensive and Third-Party}
\end{axis}
\end{tikzpicture}
\caption{$E(\kappa)$ as a function of $\kappa$ for Selected Features}
\label{kap}
\end{figure}
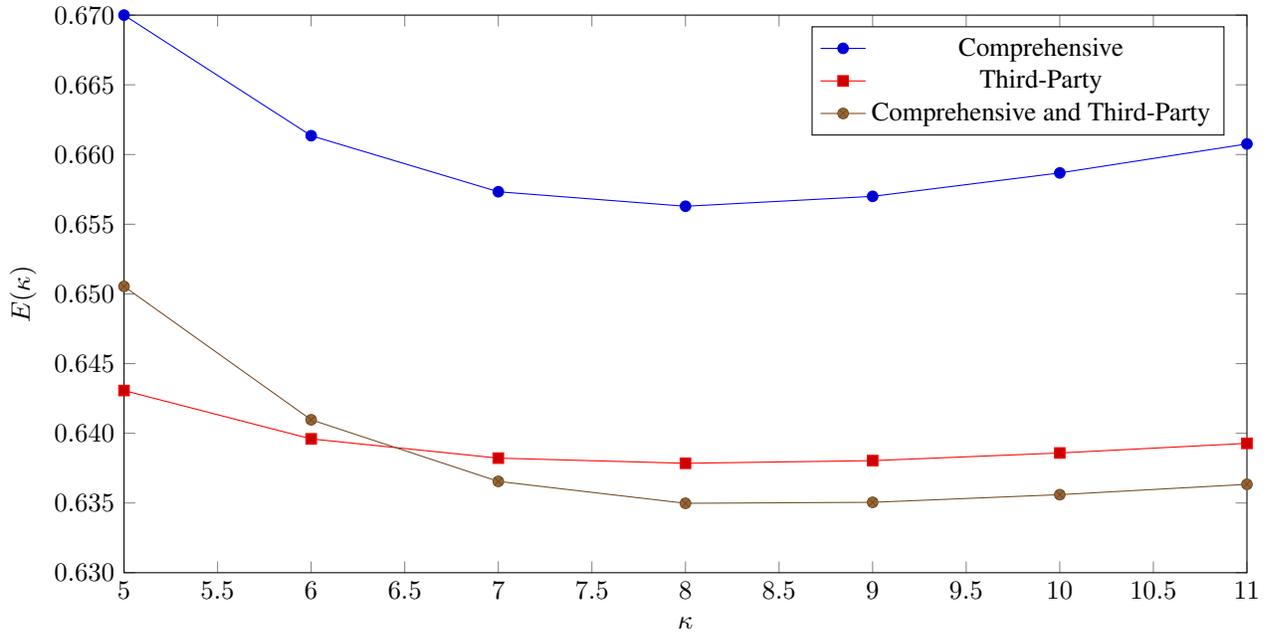

\section{Illustrative Examples of Predictions}

We have now determined the features to be used and the optimal value of $\kappa$. In this section we will consider various policy scenarios and predict the resulting claim rate to demonstrate the dependence on the features. Although we use both Comprehensive and Third Party policies in our model we will illustrate using a Comprehensive policy and normalize claim rates with respect to the average over all Comprehensive Policies. In Table \ref{cases} we provide various scenarios to illustrate that the model provide reasonable outputs. The top table starts with a low risk policy and features are changed one at a time that results in increased claim rates. The bottom table starts with a high risk policy and features are adjusted one at a time in order to lower the claim rate.

There is one outstanding case that provided unexpected results. In the lower table when we reduce the number of not-at-fault accidents from 1 to zero we expect a decrease in the claim rate but instead we found that it increased. We investigated this in detail. We found that the provided data has some inconsistencies. There were many cases where the number of not-at-fault accidents and at-fault accidents were 0 but the driver indicated that they made a claim within the last five years. This of course is inconsistent. This lead to a lot of claims listed under NNF=0 when they should be listed under NNF=1 or above. We believe this to be the reason for the result obtained. Our intent was not to make any adjustments to the given data for this paper to avoid any appearance of data tweaking. However, in the future, we will investigate what happens when we adjust the data to make it more consistent while justifying any changes made.

\begin{table}
\centering
\caption{Predictions for Sample Cases starting with Low Risk Case (top) and High Risk Case (bottom)}
\label{cases}
\setlength{\tabcolsep}{10pt}
\renewcommand{\arraystretch}{1.5}
{\small 
\begin{tabular}{|c|c|c|c|c|c|c|c|c|} 
\hline
{\bf DAF} & {\bf NAF} & {\bf NNF} & {\bf TOC} & {\bf YCF} & {\bf COV} & {\bf USE} & {\bf SEX} & {\bf Claim Rate} \\ \hline
15 & 0 & 0 &  CM & 15 & Y & Business & M & 0.07 \\ \hline
6  & 0 & 0 &  CM & 15 & Y & Business & M & 0.63 \\ \hline
2  & 1 & 0 &  CM & 15 & Y & Business & M & 0.93  \\ \hline
15 & 0 & 1 &  CM & 15 & Y & Business & M & 0.68  \\ \hline
15 & 0 & 0 &  TP & 15 & Y & Business & M & 0.03  \\ \hline
15 & 0 & 0 &  CM & 0  & Y & Business & M & 0.20 \\ \hline
15 & 0 & 0 &  CM & 15 & N & Business & M & 0.07 \\ \hline
15 & 0 & 0 &  CM & 15 & Y & Private & M & 0.06  \\ \hline
15 & 0 & 0 &  CM & 15 & Y & Business & F & 0.07 \\ \hline
\end{tabular}
\\[12pt]
\begin{tabular}{|c|c|c|c|c|c|c|c|c|} 
\hline
{\bf DAF} & {\bf NAF} & {\bf NNF} & {\bf TOC} & {\bf YCF} & {\bf COV} & {\bf USE} & {\bf SEX} & {\bf Claim Rate} \\ \hline
1 & 1 & 1 &  CM & 1 & N & Private & F & 6.04  \\ \hline
5 & 1 & 1 &  CM & 1 & N & Private & F & 1.53  \\ \hline
6 & 0 & 1 &  CM & 1 & N & Private & F & 0.56  \\ \hline
1 & 1 & 0 &  CM & 1 & N & Private & F & 7.59  \\ \hline
1 & 1 & 1 &  TP & 1 & N & Private & F & 4.04  \\ \hline
1 & 1 & 1 &  CM & 15 & N & Private & F & 5.55  \\ \hline
1 & 1 & 1 &  CM & 1 & Y & Private & F & 5.97  \\ \hline
1 & 1 & 1 &  CM & 1 & N & Business & F & 6.00  \\ \hline
1 & 1 & 1 &  CM & 1 & N & Private & M & 6.04  \\ \hline
\end{tabular}
}
\end{table}

\section{Interpreting Prediction}

Once a predicted claim rate is computed then this information can be used to compute a premium. The premium will take into account the operational costs of the company as well as the desired profit margin. This is another interesting area of research but is outside the scope of this paper.
Once a premium is computed it is important to explain the reason for the amount (i.e, interpretability). The operational cost and profit is independent of the customer so the only customer dependent factor is the predicted claim rate. We can determine the influence of each feature on this claim rate and this information can be used to explain the decision made. We do this as follows. Consider any feature $f$ and let $v$ represent the category value of this feature for the new policy. We can use the model, with only feature $f$, to determine the predicted claim rate for anyone in category $v$. Let us denote this predicted claim rate of this feature by $\tilde{C}(f, v)$. Note that this is not the same as the average claim rate over all training samples with category value $v$ which we previously denoted by $C(f, v)$. Let us explain with the feature gender. If $\kappa=0$ then $\tilde{C}({\sf gender, male}) = \tilde{C}({\sf gender, female}) = \bar{c}$. However as $\kappa$ goes to infinity then $\tilde{C}({\sf gender, male})$ approaches $C({\sf gender, male})$ and $\tilde{C}({\sf gender, female})$ approaches $C({\sf gender, female})$. For positive values of $\kappa$, $\tilde{C}(f, v)$ will lie between $\bar{c}$ and $C(f, v)$.

The metric $I_f \equiv \tilde{C}(f, v)/\bar{c}$ will be used to represent the impact of feature $f$ (of a policy with value $v$ for the feature) where $\bar{c}$ is the average claim rate. For this exercise we only use Third-Party samples to better explain the approach. If $I_f < 1$ then the feature is causing a reduction of the claim rate otherwise it is causing an increase in the claim rate. Note that all values are being computed using $\kappa = \kappa^*$ and normalized with respect to the average claim rate for Third-Party policies.

For the policy we chose we have the following information. The driver got into an accident and made a claim 9 years ago ($I_{DAF} > 1$). They have no at-fault accidents over the last 5 years ($I_{NAF} < 1$). They have had 1 not-at-fault accidents over the last five years ($I_{NNF} > 1$). They have a Third-Party Policy (hence $I_{TOC} = 1$). Their car was last in an accident 18 years ago ($I_{YCF} < 1$). They have been continuously insured over the last five years ($I_{COV} < 1$). This is their private vehicle ($I_{USE} < 1$). The driver is Male ($I_{SEX} < 1$ but almost 1). The predicted claim rate (which was impacted by the various features) has $I_{CLR} = 1.0$. We provide this information visually in Figure \ref{cont}. Hence the provider can explain to the customer the specific reasons for the premium of their policy.

\begin{figure}
\caption{Contribution of each Feature to Prediction}
\label{cont}
\centering
\begin{tikzpicture}
\begin{axis} [
ybar,
ticklabel style = {font=\small},
width=\textwidth, height=8cm,
symbolic x coords={DAF, NAF, NNF, TOC, YCF, COV, USE, SEX, CLR}, 
xtick={DAF, NAF, NNF, TOC, YCF, COV, USE, SEX, CLR}, ytick style={draw=none},
bar shift=0pt,
]
\addplot coordinates {(DAF, 2.82)};
\addplot coordinates {(NAF, 0.24)};
\addplot coordinates {(NNF, 2.09)};
\addplot coordinates {(TOC, 1.00)};
\addplot coordinates {(YCF, 0.375)};
\addplot coordinates {(COV, 0.706)};
\addplot coordinates {(USE, 1.03)};
\addplot coordinates {(SEX, 1.00)};
\addplot coordinates {(CLR, 1.00)};

\coordinate (A) at (axis cs:DAF,1);
\coordinate (O1) at (rel axis cs:0,0);
\coordinate (O2) at (rel axis cs:1,0);
\draw [red,sharp plot] (A -| O1) -- (A -| O2);

\end{axis}
\end{tikzpicture}
\end{figure}
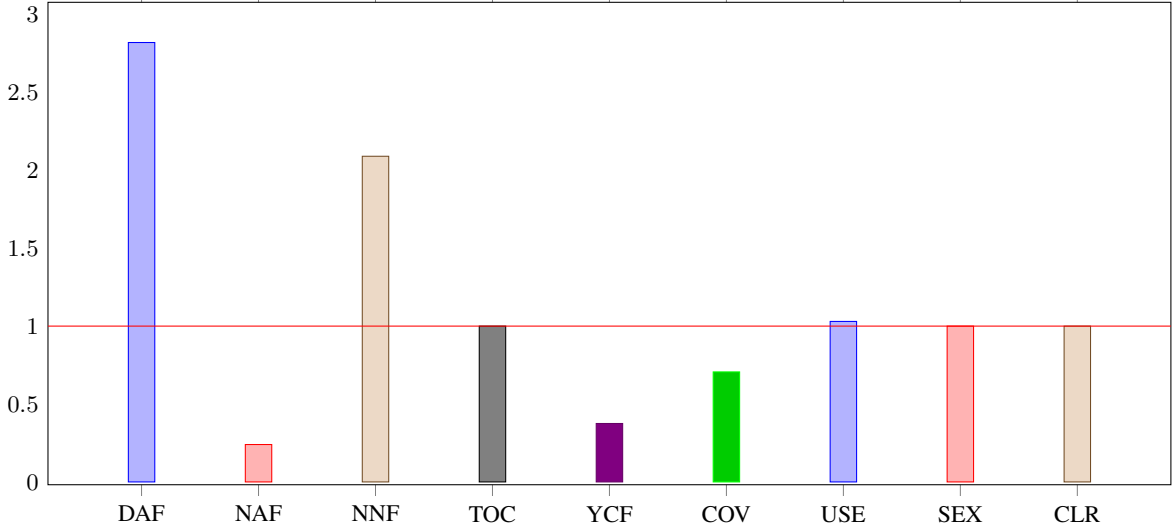

\section{Some Analytical Results}

\subsection{Expected value of Prediction}

Consider the predicted claim rate $\hat{c}$ for some test sample. If we assume that the training samples have an average claim rate $\bar{c}$ then we will show that the expected value of $\hat{c}$ is $\bar{c}$. This would mean that the sum of predicted claim rates approaches the sum of the actual claim rates as the number of test samples increases. This ensures that, at the end of the year, the total claims that are predicted is close to the total actual claims that were made.
\begin{lemma}
If the training samples have a mean claim rate of $\bar{c}$ then the expected value of the prediction for a test sample is equal to $\bar{c}$.
\end{lemma}

\begin{proof}
Recall that the predicted value for a given $\kappa$ is given by
\[
 \hat{c}(\kappa) = \frac{\sum_{s \in \vec{S}} \frac{c_s}{(1 + d_s)^{\kappa}}}{\sum_{s \in \vec{S}} \frac{1}{(1 + d_s)^{\kappa}}}
\]
We need to take the expectation of the right hand side. Now note that the feature values for a particular set of training samples is fixed and only the claim rate varies. Also note that the feature values of the test sample is also fixed. This means that $d_s$ (for sample $s$) is fixed given the specific test and training samples. This means that the denominator is constant and so we have
\[
{\cal E}[\hat{c}(\kappa)] = \frac{\sum_{s \in \vec{S}} \frac{{\cal E}[c_s]}{(1 + d_s)^{\kappa}}}{\sum_{s \in \vec{S}} \frac{1}{(1 + d_s)^{\kappa}}}
\]
Using the fact that ${\cal E}[c_s] = \bar{c}$ we obtain ${\cal E}[\hat{c}(\kappa)] = \bar{c}$.
\end{proof}

\subsection{Limiting values of \texorpdfstring{$c(\kappa)$}{}}

We have seen that $\hat{c}(\kappa=0) = \bar{c}$. In this section we compute the limit of $\hat{c}(\kappa)$ as $\kappa$ tends to $\infty$. We then take the expected value of this limit.

\begin{lemma}
If the training samples have a mean claim rate of $\bar{c}$ then
\begin{equation}
    {\cal E}[\lim_{\kappa \to \infty}c(\kappa)] = \bar{c}
\end{equation}
\end{lemma}
\begin{proof}
We first compute the limit of the predicted value, $\hat{c}(\kappa)$ as $\kappa$ tends to infinity. There are two cases to consider. If one or more training samples have identical feature values as the test sample (i.e., $d_s = 0$) then $\hat{c}_s$ is the average of these values. However, if none of the training samples have identical features then both the numerator and denominator of $\hat{c}(\kappa)$ tend to zero as $\kappa$ tends to infinity so we instead do the following.
Denote the training sample with the smallest distance from the test sample by $s'$.  Let us multiply top and bottom of the equation for $\hat{c}(\kappa)$ by $(1 + d_{s'})^\kappa$ to obtain
\begin{equation}
    \lim_{\kappa \to \infty}c(\kappa) = \lim_{\kappa \to \infty }  \frac{c_{s'} + \sum_{s \in \vec{S}|s \neq s'} c_s \left( \frac{1+d_{s'}}{1 + d_s} \right) ^{\kappa}}{1 + \sum_{s \in \vec{S}|s \neq s'} \left( \frac{1 + d_{s'}}{1 + d_s} \right) ^{\kappa}}
\end{equation}
Now note that since $d_{s'} < d_s$ for all samples $s$ then, as $\kappa$ goes to infinity, the summations go to zero and hence
\begin{equation}
     \lim_{\kappa \to \infty}\hat{c}(\kappa) = c_{s'}
\end{equation}
If multiple samples are at this distance $d_{s'}$ then the numerator constant would be the sum of these samples and the denominator would be the number of them and hence the limit is the average of the claim rates of these samples.
Since $c_{s'}$ is a sample from the training set space then its expected value is  $\bar{c}$ and hence
\begin{equation}
    {\cal E}[\lim_{\kappa \to \infty}\hat{c}(\kappa)] = \bar{c}
\end{equation}
\end{proof}

One should note the following. As the number of training samples increases, more training samples will be available close to the test sample and hence the optimal value of $\kappa$ will increase. In the limit, the predicted value becomes the true mean for the features of the test sample thus achieving precise personalization.

\section{Conclusions and Future Work}

We presented a new model for automobile insurance risk assessment and demonstrated its effectiveness using real data. We showed how feature importance can be computed, how features can be selected and how model parameters are optimized. Finally we demonstrated how the model can be used in practice and results interpreted. Note that this approach can be applied to any regression problem and its performance will improve as the variance of the target metric decreases. 

Future work will include application of the model to other problems as well as on model improvements and increased computational efficiency. We also plan to investigate how recursive improvement of distance values can be used to increase accuracy (see \cite{hosein2021prediction} for an example of this improvement). Finally we plan to investigate properties of the model such as whether $E(\kappa)$ is convex when its gradient at $\kappa=0$ is negative.

\bibliographystyle{spbasic.bst}
\bibliography{claims}

\end{document}